%% file: main.tex
\pdfoutput=1
\documentclass{article}

\usepackage[final]{neurips_2020}

\usepackage[utf8]{inputenc} % allow utf-8 input
\usepackage[T1]{fontenc}    % use 8-bit T1 fonts
\usepackage{hyperref}       % hyperlinks
\usepackage{url}            % simple URL typesetting
\usepackage{booktabs}       % professional-quality tables
\usepackage{amsfonts}       % blackboard math symbols
\usepackage{nicefrac}       % compact symbols for 1/2, etc.
\usepackage{microtype}      % microtypography
\usepackage{grffile}
\usepackage{amsmath}
\input{math_commands}

\usepackage{listings}
\usepackage[frozencache,cachedir=minted]{minted}
\usepackage{color}
\usepackage{todonotes}
\usepackage{wrapfig}
\usepackage{multirow}
\usepackage{enumitem}
\usepackage[ruled,vlined]{algorithm2e}

\usepackage{graphicx}
\usepackage{bmpsize}

\usepackage{natbib}
\newtheorem{proposition}{Proposition}
\newtheorem{proof}{Proof}

\newcommand\Item[1][]{%
  \ifx\relax#1\relax  \item \else \item[#1] \fi
  \abovedisplayskip=0pt\abovedisplayshortskip=0pt~\vspace*{-\baselineskip}}

\renewrobustcmd{\bfseries}{\fontseries{b}\selectfont}
\renewcommand{\pm}{\mathbin{\mbox{\unboldmath$\mathchar"2206$}}}

\title{Deep Reinforcement and InfoMax Learning}

\author{%
  Bogdan Mazoure$^1$\thanks{Equal contribution. $^1$Work done during an internship at Microsoft Research Montr\'{e}al. Correspondence to: \texttt{bogdan.mazoure@mail.mcgill.ca}} \\
  McGill University, Mila\\
  \And
  R\'{e}mi Tachet des Combes$^{*}$ \\
  Microsoft Research Montr\'{e}al \\
  \And
  Thang Doan \\
  McGill University, Mila \\
  \And
  Philip Bachman \\
  Microsoft Research Montr\'{e}al\\
  \And
  R Devon Hjelm \\
  Microsoft Research Montr\'{e}al\\
  Universit\'{e} de Montr\'{e}al, Mila \\
}

\begin{document}

\maketitle

\begin{abstract}
    We begin with the hypothesis that a model-free agent whose representations are predictive of properties of future states (beyond expected rewards) will be more capable of solving and adapting to new RL problems. To test that hypothesis, we introduce an objective based on Deep InfoMax (DIM) which trains the agent to predict the future by maximizing the mutual information between its internal representation of successive timesteps. We test our approach in several synthetic settings, where it successfully learns representations that are predictive of the future. Finally, we augment C51, a strong RL baseline, with our temporal DIM objective and demonstrate improved performance on a continual learning task and on the recently introduced Procgen environment.
\end{abstract}

\section{Introduction}
\input{introduction}
\section{Background}
\input{related}
\section{Preliminaries}
\input{preliminary}
\section{Architecture and Algorithm}
\input{algorithm}

\section{Finding the Best Task Timescale}
\input{adaptive_action}

\section{Experiments}
\input{experiments}

\section{Discussion}
\input{discussion}

\section*{Acknowledgements}
We thank Harm van Seijen, Ankesh Anand, Mehdi Fatemi, Romain Laroche and Jayakumar Subramanian for useful feedback and helpful discussions.
\section*{Broader Impact}
This work proposes an auxiliary objective for model-free reinforcement learning agents. The objective shows improvements in a continual learning setting, as well as on average training rewards for a suite of complex video games. While the objective is developed in a visual setting, maximizing mutual information between features is a method that can be transported to other domains, such as text. Potential applications of deep reinforcement learning are (among others) healthcare, dialog systems, crop management, robotics, etc. Developing methods that are more robust to changes in the environment, and/or perform better in a continual learning setting can lead to improvements in those various applications. 
At the same time, our method fundamentally relies on deep learning tools and architectures, which are hard to interpret and prone to failures yet to be perfectly understood. Additionally, deep reinforcement learning also lacks formal performance guarantees, and so do deep reinforcement learning agents. Overall, it is essential to design failsafes when deploying such agents (including ours) in the real world.

% \bibliography{bibliography}

\input{main.bbl}
\clearpage

\section{Appendix}
\input{appendix}
\end{document}

%% file: math_commands.tex
\usepackage{amsmath,amsfonts,bm,bbm}

\def\eqref#1{equation~\ref{#1}}
\def\1{\bm{1}}

\DeclareMathAlphabet{\mathsfit}{\encodingdefault}{\sfdefault}{m}{sl}
\SetMathAlphabet{\mathsfit}{bold}{\encodingdefault}{\sfdefault}{bx}{n}

\newcommand{\cA}{\mathcal{A}}
\newcommand{\cB}{\mathcal{B}}

\newcommand{\cF}{\mathcal{F}}

\newcommand{\cI}{\mathcal{I}}

\newcommand{\cL}{\mathcal{L}}

\newcommand{\cS}{\mathcal{S}}
\newcommand{\cT}{\mathcal{T}}

\newcommand{\Real}{\mathbb{R}}

\renewcommand{\vec}[1]{\ensuremath{\bm{#1}}}

\newcommand{\mat}[1]{\ensuremath{\mathbf{#1}}}

%% file: introduction.tex
In reinforcement learning (RL), model-based agents are characterized by their ability to predict future states and rewards based on past states and actions~\citep{SuttonBarto98,ha2018world,hafner2019dream}.
Model-based methods can be seen through the \emph{representation learning}~\citep{goodfellow2017learning} lens as endowing the agent with internal representations that are predictive of the future conditioned on its actions.
This ultimately gives the agent means to plan -- by e.g. considering a distribution of possible future trajectories and picking the best course of action.

In contrast, model-free methods do not explicitly model the environment, and instead learn a policy that maximizes reward or a function that estimates the optimal values of states and actions~\citep{mnih2015human,schulman2017ppo,pong2018temporal}.
They can use large amounts of training data and excel in high-dimensional state and action spaces. However, this is mostly true for fixed reward functions; despite success on many benchmarks, model-free agents typically generalize poorly when the environment or reward function changes~\citep{farebrother2018generalization,journals/corr/abs-1809-02591} and can have high sample complexity.

Viewing model-based agents from a representation learning perspective, a desired outcome is an agent that understands the underlying generative factors of the environment that determine the observed state/action sequences, leading to generalization to other environments built from the same generative factors. In addition, learning a predictive model affords a richer learning signal than those provided by reward alone, which could reduce sample complexity compared to model-free methods.

Our work is based on the hypothesis that a model-free agent whose representations are predictive of properties of future states (beyond expected rewards) will be more capable of solving and adapting to new RL problems and, in a way, incorporate aspects of model-based learning.
To learn representations with model-like properties, we consider a self-supervised objective derived from variants of Deep InfoMax~\citep[DIM,][]{hjelm2018learning, bachman2019learning, anand2019unsupervised}.
We expect this type of contrastive estimation~\citep{hyvarinen2016unsupervised} will give the agent a better understanding of the underlying factors of the environment and how they relate to its actions, eventually leading to better performance in transfer and lifelong learning problems.
We examine the properties of the learnt representations in simple domains such as disjoint and glued Markov chains, and more complex environments such as a 2d Ising model, a sequential variant of Ms.~PacMan from the Atari Learning Environment~\citep[ALE,][]{bellemare2013arcade}, and all 16 games from the Procgen suite~\citep{cobbe2019leveraging}.
Our contributions are as follows:
\begin{itemize}
\item We propose a simple auxiliary objective that maximizes concordance between representations of successive states, given the action. We also introduce a simple adaptive mechanism that adjusts the time-scales of the contrastive tasks based on the likelihood of subsequent actions under the current RL policy.

\item We present a series of experiments showing how our objective can be used as a measure of similarity and predictability, and how it behaves in partially deterministic systems.

\item Finally, we show that augmenting a standard RL agent with our contrastive objective can i) lead to faster adaptation in a continual learning setting, and ii) improve overall performance on the Procgen suite.
\end{itemize}

%% file: related.tex
Just as humans are able to retain old skills when taught new ones~\citep{wixted2004psychology}, we strive for RL agents that are able to adapt quickly and reuse knowledge when presented a sequence of different tasks with variable reward functions. 
The reason for this is that real-world applications or downstream tasks can be difficult to anticipate before deployment, particularly with complex environments involving other intelligent agents such as humans.
Unfortunately, this proves to be very challenging even for state-of-the-art systems~\citep{atkinson2018pseudo}, leading to complex deployment scenarios.

Continual Learning (CL) is a learning framework meant to benchmark an agent's ability to adapt to new tasks by using auxiliary information about the relatedness across tasks and timescales~\citep{kaplanis2018continual,mankowitz2018unicorn,doan2020theoretical}.
Meta-learning~\citep{thrun1998learning,finn2017model} and multi-task learning~\citep{hessel2019multi,d2019sharing} have shown good performance in CL by explicitly training the agent to transfer well between tasks. 

In this study, we focus on the following inductive bias: while the reward function may change or vary, the underlying environment dynamics typically do not change as much\footnote{This is not true in all generalization settings. Generalization still has a variety of specifications within RL. In our work, we focus on the setting where the rewards change more rapidly than the environment dynamics.}. 
To test if that inductive bias is useful, we use \emph{auxiliary loss functions} to encourage the agent to learn about the underlying generative factors and their associated dynamics in the environment, which can result in better sample efficiency and transfer capabilities (compared to learning from rewards only). Previous work has shown this idea to be useful when training RL agents: e.g., \cite{jaderberg2016reinforcement} train the agent to predict future states given the current state-action pair, while \cite{mohamed2015variational} uses empowerment to measure concordance between a sequence of future actions and the end state.
Recent work such as DeepMDP~\citep{gelada2019deepmdp} uses a latent variable model to represent transition and reward functions in a high-dimensional abstract space. In model-based RL, various agents, such as PlaNet~\citep{hafner2019learning}, Dreamer~\citep{hafner2019dream}, or MuZero~\citep{schrittwieser2019mastering}, have also shown strong asymptotic performance.

% Contrastive learning is a representation learning method that trains an encoder model to carry information that is similar across different ``views'' of the data in its high-level feature space. 
Contrastive representation learning methods are based on training an encoder to capture information that is shared across different views of the data in the features it produces for each input. The similar (i.e. positive) examples are typically either taken from different ``locations'' of the data~\citep[e.g., spatial patches or temporal locations, see][]{hjelm2018learning, oord2018representation, anand2019unsupervised, henaff2019data} or obtained through data augmentation~\citep{wu2018unsupervised, he2019momentum, bachman2019learning, tian2019contrastive, chen2020simple}.
Contrastive models rely on a variety of objectives to encourage similarity between features. 
Typically, a scoring function~\citep[e.g., dot product or cosine similarity between pairs of features, see][]{wu2018unsupervised} that lower-bounds mutual information is maximized~\citep{belghazi2018mine, hjelm2018learning, oord2018representation, poole2019variational}.

A number of works have applied the above ideas to RL settings.
Contrastive Predictive Coding \citep[CPC,][]{oord2018representation} augments an A2C agent with an autoregressive contrastive task across a sequence of frames, improving performance on 5 DeepMind lab games~\citep{beattie2016deepmind}. 
EMI~\citep{kim2019emi} uses a Jensen-Shannon divergence-based lower bound on mutual information across subsequent frames as an exploration bonus.
CURL~\citep{srinivas2020curl} uses a contrastive task using augmented versions of the same frame (does not use future frames) as an auxiliary task to an RL algorithm. 
Finally, HOMER~\citep{misra2019kinematic} produces a policy cover for block MDPs by learning backward and forward state abstractions using contrastive learning objectives. It is worth noting that HOMER has statistical guarantees for its performance on certain hard exploration problems.

Our work, DRIML, predicts future states conditioned on the current state-action pair at multiple scales, drawing upon ideas encapsulated in Augmented Multiscale Deep InfoMax~\citep[AMDIM,][]{bachman2019learning} and Spatio-Temporal DIM~\citep[ST-DIM,][]{anand2019unsupervised}.
Our method is flexible w.r.t. these tasks: we can employ the DIM tasks over features that constitute the full frame (global) or that are specific to local patches (local) or both. It is also robust w.r.t. time-scales of the contrastive tasks, though we show that adapting this time scale according to the predictability of subsequent actions under the current RL policy improves performance substantially.

%% file: preliminary.tex
We assume the usual Markov Decision Process (MDP) setting (see Appendix for details), with the MDP denoted as $\mathcal{M}$, states as $s$, actions as $a$, and the policy as $\pi$.
Since we focus on exploring the role of auxiliary losses in continuous learning, we use C51~\citep{bellemare2017distributional}, which extends DQN~\citep{mnih2015human} to predict the full distribution of potential future rewards, for training the agent due to its strong performance on control tasks from pixels. C51 minimizes the following loss:
\begin{equation}
    \cL_{RL}=D_{KL}(\lfloor \cT Z(s,a)\rfloor_{51} ||Z(s,a)),
    \label{eq:C51_loss}
\end{equation}
where $D_{KL}$ is the Kullback-Leibler divergence, $Z(s,a)$ is the distribution of future discounted returns under the current policy ($\mathbb{E}[Z(s,a)]=Q(s,a)$), $\cT$ is the distributional Bellman operator~\citep{bellemare2019distributional} and $\lfloor \cdot \rfloor_{51}$ is an operator which projects $Z$ onto a fixed support of $51$.

\subsection{State-action mutual information maximization}
Mutual information (MI) measures the amount of information shared between a pair of random variables and can be estimated using neural networks~\citep{belghazi2018mine}.
Recent representation learning algorithms~\citep{oord2018representation, hjelm2018learning, tian2019contrastive, he2019momentum} train encoders to maximize the MI between features taken from different views of the input  -- e.g., different patches in an image, different timesteps in a sequence, or different versions of an image produced by applying data augmentation to it.
%These algorithms commonly optimize lower bounds on the MI based on noise-contrastive estimation~\citep[NCE,][]{gutmann2010noise, oord2018representation}.

Let $k$ be some fixed temporal offset. Running a policy $\pi$ in the MDP $\mathcal{M}$ generates a distribution over tuples $(s_t,a_t,s_{t+k})$, where $s_t$ corresponds to the state of $\mathcal{M}$ at some timestep $t$, $a_t$ to the action selected by $\pi$ in state $s_t$ and $s_{t+k}$ to the state of $\mathcal{M}$ at timestep $t+k$, reached by following $\pi$. $S_t$, $A_t$ and $S_{t+k}$ stand for the corresponding random variables. We also denote the joint distribution of these variables, as well as their associated marginals, using $p$.
% Consider a trajectory in our MDP, generated by following a policy $\pi$: $\xi = (s_0, a_0), \dots, (s_t, a_t), \dots$, where $0 \leq t < T$. Let $k = t' - t$ be some fixed offset. 
% We draw a view of trajectory $\xi$ by selecting a tuple $(s_t, a_t, s_{t'})$, corresponding to a state-action pair at time $t$ and the state on the same trajectory $k$ timesteps later.
% Let us denote $X_\xi = \{(s_t, a_t, s_{t'}) | 0 \leq t < (T - k) \}$ as the set of all possible tuples drawn in this manner from $\xi$.
We are interested in learning representations of state-action pairs that have high MI with the representation of states later in the trajectory. The MI between e.g. state-action pairs $(S_t,A_t)$ and their future states $S_{t+k}$ is defined as follows:
\begin{equation}
    \cI([S_t, A_t], S_{t+k}; \pi) = \mathbb{E}_{p_{\pi}(s_t,a_t,s_{t+k})} \bigg[\log\frac{p_{\pi}(s_t,a_t,s_{t+k})}{p_{\pi}(s_t,a_t)p_{\pi}(s_{t+k})}\bigg],
    \label{eq:MI}
\end{equation}
where $p_{\pi}$ denotes distributions under $\pi$. Estimating the MI can be done by training a classifier that discriminates between a sample drawn from the joint distribution -- the numerator of Eq.~\ref{eq:MI} -- and a sample from the product of marginals -- its denominator. A sample from the product of marginals is usually obtained by replacing $s_{t+k}$ (positive sample) with a state picked at random from another trajectory (negative sample). Letting $S^{-}$ denote a set of such negative samples, the infoNCE loss function~\citep{gutmann2010noise, oord2018representation} that we use to maximize a lower bound on the MI in Eq.~\ref{eq:MI} (with the added encoders for the states and actions) takes the following form:
\begin{equation}
    \cL_{NCE} :=- \mathbb{E}_{p_{\pi}(s_t, a_t, s_{t+k})} \mathbb{E}_{S^{-}} \bigg[ \log \frac{\exp(\phi(\Psi(s_t, a_t), \Phi(s_{t+k})))}{\sum_{s^{\prime} \in S^{-} \cup \{s_{t+k}\}} \exp(\phi(\Psi(s_t, a_t), \Phi(s^{\prime})))} \bigg],
    \label{eq:original_nce}
\end{equation}
where $\Psi(s, a), \Phi(s)$ are features that depend on state-action pairs and states, respectively, and $\phi$ is a function that outputs a scalar-valued \emph{score}.
Minimizing $ \cL_{NCE}$ with respect to $\Phi, \Psi$, and $\phi$ maximizes the MI between these features.
In practice, we construct $S^{-}$ by including all states $\tilde{s}_{t+k}$ from other tuples $(\tilde{s}_t, \tilde{a}_t, \tilde{s}_{t+k})$ in the same minibatch as the relevant $(s_t, a_t, s_{t+k})$. I.e., for a batch containing $N$ tuples $(s_t, a_t, s_{t+k})$, each $S^{-}$ would contain $N-1$ negative samples.

%% file: algorithm.tex
\label{sec:algo}

We now specify forms for the functions $\Phi, \Psi$, and $\phi$. We consider a deep neural network $\Theta : \cS \to \prod_{i=1}^5 \cF_i$ which maps input states onto a sequence of progressively more ``global'' (or less ``local'') feature spaces.
In practice, $\Theta$ is a CNN composed of functions that sequentially map inputs to features $\{f_i \in \cF_i\}_{1\leq i \leq 5}$ (lower to upper ``levels'' of the network). For ease of explanation, we formulate our model using specific features (e.g., \emph{local} features $f_3$ and \emph{global} features $f_4$), but our model covers any set of features extracted from $\Theta$ used for the objective below as well as other choices for $\Theta$.

The features ${f_5}$ are the output of the network's last layer and correspond to the standard C51 value heads (i.e., they span a space of 51 atoms per action)~\footnote{$\Theta$ can equivalently be seen as the network used in a standard C51.}.
For the auxiliary objective, we follow a variant of Deep InfoMax~\citep[DIM,][]{hjelm2018learning, anand2019unsupervised,bachman2019learning}, and train the encoder to maximize the mutual information (MI) between local and global ``views'' of tuples $(s_t, a_t, s_{t+k})$.
The local and global views are realized by selecting $f_3 \in \cF_3$ and $f_4 \in \cF_4$ respectively.
In order to simultaneously estimate and maximize the MI, we embed the action (represented as a one-hot vector) using a function $\Psi_a: \cA \to \tilde{\cA}$. We then map the local states $f_3$ and the embedded action using a function $\Psi_3:\cF_3\times \tilde{\cA} \to L$, and do the same with the global states $f_4$, i.e., $\Psi_4:\cF_4 \times \tilde{\cA} \to G$. In addition, we have two more functions, $\Phi_3:\cF_3 \to L$ and $\Phi_4:\cF_4 \to G$ that map features without the actions,  which will be applied to features from ``future'' timesteps. Note that $L$ can be thought of as a product of local spaces (corresponding to different patches in the input, or equivalently different receptive fields), each with the same dimensionality as $G$.

We use the outputs of these functions to produce a scalar-valued score between any combination of local and global representations of state $s_t$ and $s_{t+k}$, conditioned on action $a_t$:
\begin{align}
    \phi_{NtM}(s_t, a, s_{t+k}) &:= \Psi_N(f_N(s_t), \Psi_a(a_t))^\top \Phi_{M}(f_M(s_{t+k})), \; M, N\in\{3, 4\}. \label{eq:nce_scores}
\end{align}

In practice, for the functions that take features and actions as input, we simply concatenate the values at position $f_3$ (local) or $f_4$ (global) with the embedded action $\Psi_a(a)$, and feed the resulting tensor into the appropriate function $\Psi_{3}$ or $\Psi_{4}$. All functions that process global and local features are computed using $1\times1$ convolutions. See Figure~\ref{fig:architecture} for a visual representation of our model.

We use the scores from Eq.~\ref{eq:nce_scores} when computing the infoNCE loss~\citep{oord2018representation} for our objective, using $(s_t, a_t, s_{t+k})$ tuples sampled from trajectories stored in an experience replay buffer:
 \begin{equation}
     \cL_{DIM}^{NtM}:=-\mathbb{E}_{p_{\pi}(s_t, a_t, s_{t+k})} \mathbb{E}_{S^{-}} \bigg[ \log \frac{\exp(\phi_{NtM}(s_t, a_t, s_{t+k}))}{\sum_{s^{\prime} \in S^{-} \cup \{s_{t+k}\}} \exp(\phi_{NtM}(s_t, a_t, s^{\prime}))} \bigg]\;.
     \label{eq:average_infonce}
 \end{equation}

Combining Eq.~\ref{eq:average_infonce} with the RL update in Eq.~\ref{eq:C51_loss} yields our full training objective, which we call DRIML~\footnote{Deep Reinforcement and InfoMax Learning}. We optimize $\Theta, \Psi_{3,4,a}$, and $\Phi_{3,4}$ jointly using a single loss function:
 \begin{equation}
   \cL_{DRIML} = \cL_{RL}+\sum_{M,N\in \{3,4\}}\lambda_{NtM} \cL^{NtM}_{DIM}
     \label{eq:composite_loss}
 \end{equation}

Note that, in practice, the compute cost which Eq.~\ref{eq:composite_loss} adds to the core RL algorithm is minimal, since it only requires additional passes through the (small) state/action embedding functions followed by an outer product.

 \begin{algorithm}[H]
 \SetAlgoLined
 \SetKwInOut{Input}{Input}
 \SetKwInOut{Output}{Output}
 \Input{Batch $\cB$ sampled from the replay buffer, $\{\lambda_{NtM}\}_{N,M\in\{3,4\}}$, strictly positive integer $k$}
 Update $\Theta$ using Eq.~\ref{eq:C51_loss}\;
 $s,a,s',x\leftarrow \cB[s_t],\cB[a_t],\cB[s_{t+k}],\cB[s_{t'\neq t+k}]$\;
 \For{N in $\{$3,4$\}$}{
 \For{M in $\{$3,4$\}$}{
 \If{$\lambda_{NtM}>0$}{
 Compute $\cL_{DIM}^{NtM}$ using Eq.~\ref{eq:average_infonce} (see Appendix~\ref{code:dim_batch_scores} for PyTorch code)\;
 Update $\Theta, \Psi_{3,4,a}$, and $\Phi_{3,4}$ using gradients of $\lambda_{NtM}\cL_{DIM}^{NtM}$\;
 }
 }
 }
 \label{algo:main}
  \caption{Deep Reinforcement and InfoMax Learning (DRIML)}
 \end{algorithm}

 \begin{figure}[h!]
     \begin{center}
         \includegraphics[width=\linewidth]{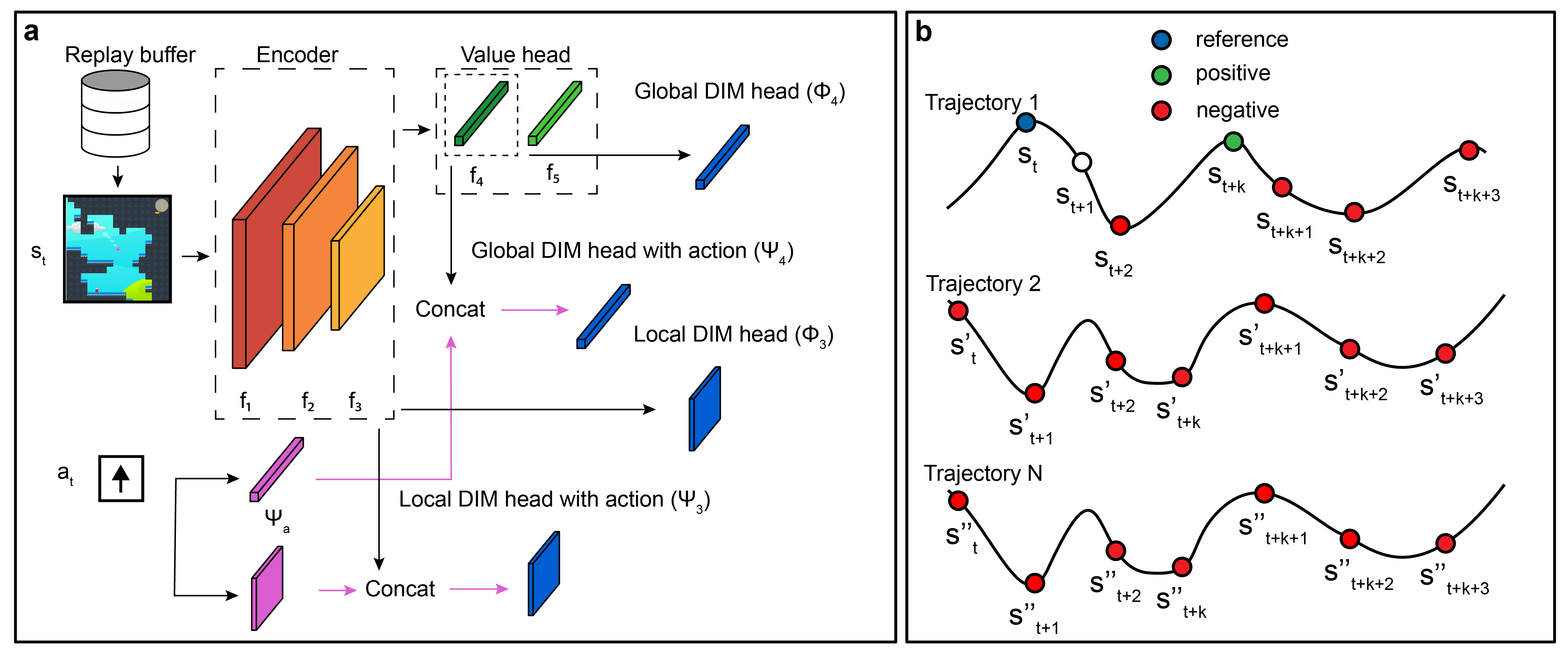}
         \caption{\textbf{(a)} Example model architecture used for the encoder used for the RL and DIM objectives and \textbf{(b)} distribution of reference, positive and negative samples within training batch $\cB$. Note that in our experiments, we either use only the local head, only the global head, or both.}
     \end{center}
     \label{fig:architecture}
\end{figure}

The proposed Algorithm~\ref{algo:main} introduces an auxiliary loss which improves predictive capabilities of value-based agents by boosting similarity of representations close in time.

%% file: adaptive_action.tex
The above DRIML algorithm fixes the temporal offset for the contrastive task, $k$, which needs to be chosen a-priori.
However, different games are based on MDPs whose dynamics operate at different timescales, which in turn means that the difficulty of predictive tasks across different games will vary at different scales.
We could predict simultaneously at multiple timescales~\citep[as in][]{oord2018representation}, yet this introduces additional complexity that could be overcome by simply finding the \emph{right} timescale.
In order to ensure our auxiliary loss learns useful insights about the underlying MDP, as well as make DRIML more generally useful across environments, we adapt the temporal offset $k$ automatically based on the distribution of the agent's actions.

We propose to select an adaptive, game-specific look-ahead value $k$, by learning a function $q_{\pi}(a_i,a_j)$ which measures the log-odds of taking action $a_j$ after taking action $a_i$  when following policy $\pi$ in the environment (i.e. $p_\pi(A_{t+1} = a_j | A_t = a_i) / p_\pi(A_{t+1} = a_j)$).
The values $q_{\pi}(a_i,a_j)$ are then used to sample a look-ahead value $k\in \{1,..,H\}$ from a non-homogeneous geometric (NHG) distribution. This particular choice of distribution was motivated by two desirable properties of NHG: \textbf{(a)} any discrete positive random variable can be represented via NHG~\citep{mandelbaum2007nonhomogeneous} and \textbf{(b)} the expectation of $X\sim NHG(q_1,..,q_H)$ obeys the rule $1/\max_i q_i\leq \mathbb{E}[X]\leq 1/\min_i q_i$.
The intuition is that, if the state dynamics are regular, this should be reflected in the predictability of the future actions conditioned on prior actions. Our algorithm captures this through a matrix $\mat{A}$, whose $i$-th row is the softmax of $q_{\pi}(a_i,\cdot)$.
$q_{\pi}(a_i,a_j)$ is learned off-policy using the same data from the buffer as the main algorithm; it is updated at the same frequency as the main DRIML parameters and trained to approximate the relevant log odds.
This modification is small in relation to the DRIML algorithm, but it substantially improves results in our experiments.
The sampling of $k$ is done via Algorithm~\ref{algo:adapt}, and additional analysis of the adaptive temporal offset is provided in Figure~\ref{fig:adaptive_procgen}.

 \begin{algorithm}[H]
 \SetAlgoLined
 \SetKwInOut{Input}{Input}
 \SetKwInOut{Output}{Output}
 \Input{Tuple $(s_t,a_t,a_{t+1},...,a_{t+H})$, maximal horizon $H$, stochastic matrix $\mat{A}$ of size $\mathcal{A}\times \mathcal{A}$}
 \Output{Lookahead value $k:1\leq k\leq H$}
 
 \For{$i$ in $1,...,H$}{
 $k \leftarrow i$;
 \tcp{Updating value of k}
 
 $b\sim \text{Bernoulli}(\mat{A}_{a_{t+i-1},a_{t+i}})$;
 
 \If{$b==0$}{
 break;
 }
 }
 \label{algo:adapt}
  \caption{Adaptive lookahead selection}
 \end{algorithm}
 
 Figure~\ref{fig:adaptive_procgen} shows the impact of adaptively selecting $k$ using the NHG sampling method. For instance, (i) depending on the nature of the game, DRIML-ada tends to repeat movement actions in navigation games and repeatedly fire in shooting games, and (ii) the value of $k$ tends to converge to 1 for games like Bigfish and Plunder as training progresses, which hints to an exploration-exploitation like trade-off. 
 
 Since many Procgen games do not have special actions such as fire or diagonal moves, DRIML-ada considers the actual actions (15 of them) and the visible actions (at most 15 of them) together in the adaptive lookahead selection algorithm.
\begin{figure}[h!]
    \centering
    \includegraphics[width=\linewidth]{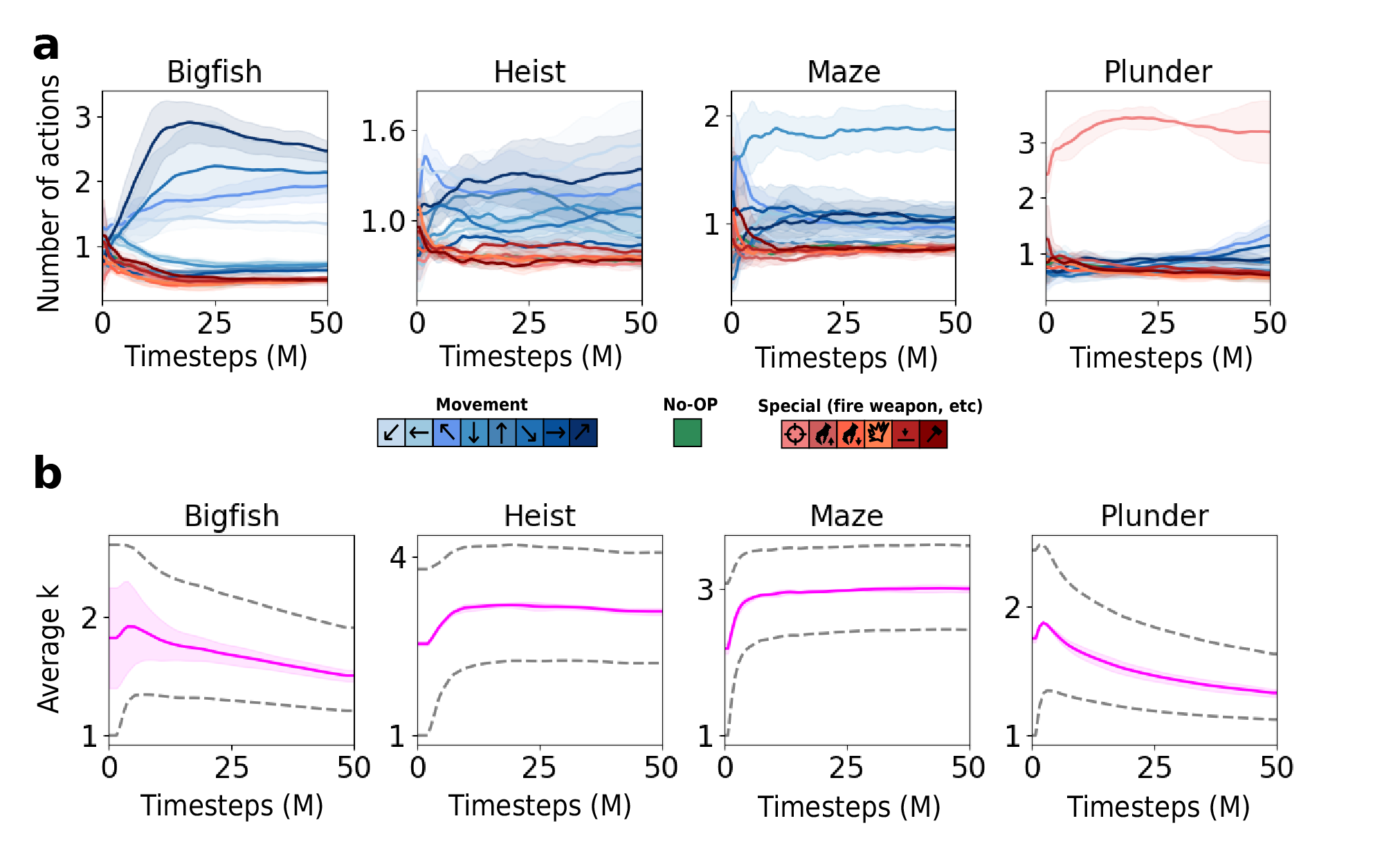}
    \caption{\textbf{(a)} Average number of movement, no-op and special actions taken by DRIML-ada on 4 Procgen games and \textbf{(b)} change in the average, max and min across batch values of $k$ as a function of training steps.}
    \label{fig:adaptive_procgen}
\end{figure}

%% file: experiments.tex
In this section, we first show how our proposed objective can be used to estimate state similarity in single Markov chains. We then show that DRIML can capture dynamics in locally deterministic systems (Ising model), which is useful in domains with partially deterministic transitions. We then provide results on a continual version of the Ms.~PacMan game where the DIM loss is shown to converge faster for more deterministic tasks, and to help in a continual learning setting. Finally, we provide results on Procgen~\citep{cobbe2019leveraging}, which show that DRIML performs well when trained on 500 levels with fixed order. All experimental details can be found in Appendix~\ref{sec:appendix_experiment_details}.

\subsection{DRIML learns a transition ratio model}
\label{sec:single_MC}
We first study the behaviour of DRIML's loss on a simple Markov chain describing a biased random walk in $\{1,\cdots,K\}$. The bias is specified by a single parameter $\alpha$. The agent starting at state $i$ transitions to $i+1$ with probability $\alpha$ and to $i-1$ otherwise. The agent stays in states $1$ and $K$ with probability $1-\alpha$ and $\alpha$, respectively. We encode the current and next states (represented as one-hots) using a 1-hidden layer MLP\footnote{The action is simply ignored in this setting.} (corresponding to $\Psi$ and $\Phi$ in~\eqref{eq:original_nce}), and then optimize the NCE loss $\cL_{DIM}$ (the scoring function $\phi$ is also 1-hidden layer MLP,~\eqref{eq:original_nce}) to maximize the MI between representations of successive states. Results are shown in Fig.~\ref{fig:fig1_random_walk_p}b, they are well aligned with the true transition matrix (Fig.~\ref{fig:fig1_random_walk_p}c). 

This toy experiment revealed an interesting observation: DRIML's objective involves computing the ratio of the Markov transition operator over the stationary distribution, implying that the convergence rate is affected by the \emph{spectral gap} of the average Markov transition operator, $\mat{T}^\pi_{ss'}=\mathbb{E}_{a \sim \pi(s,\cdot)}[T(s, a, s')]$ for transition operator $T$. That is, it depends on the difference between the two largest eigenvalues of $\mat{T}^\pi$, namely $1$ and $\lambda_{(2)}$.
In the case of the random walk, the spectral gap of its transition matrix can be computed in closed-form as a function of $\alpha$. Its lowest value is reached in the neighbourhood $\alpha=0.5$, corresponding to the point where the system is least predictable (as shown by the mutual information, Fig~\ref{fig:fig1_random_walk_p}c). However, since the matrix is not invertible for $\alpha=0.5$, we consider $\alpha=0.499$ instead. Derivations are available in Appendix~\ref{sec:appendix_single_mc}, and more insights on the connection to the spectral gap are presented in Appendix~\ref{sec:predictability_nce}.

\begin{figure*}[h!]
    \centering
    \includegraphics[width=\linewidth]{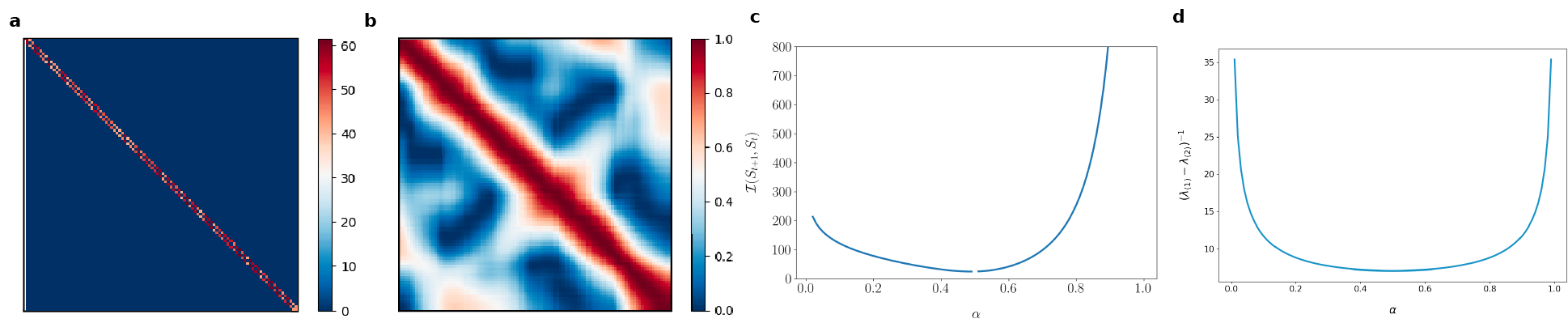}
    \caption{\textbf{(a)} Ratio of transition matrix over stationary vector for the random walk with $\alpha=0.499$, \textbf{(b)} the prediction matrix of being a pair of successive states learnt by $\cL_{DIM}$, \textbf{(c)} the closed-form mutual information between consecutive states in time as a function of $\alpha$ (with simplified endpoint conditions) and \textbf{(d)} the true inverse spectral gap $(\lambda_{(1)}-\lambda_{(2)})^{-1}$ as a function of $\alpha$.}
    \label{fig:fig1_random_walk_p}
\end{figure*}

\subsection{DRIML can capture complex partially deterministic dynamics}
The goal of this experiment is to highlight the predictive capabilities of our DIM objective in a partially deterministic system. We consider a dynamical system composed of $N\times N$ pixels with values in $\{-1,1\}$, $S(t) = \{s_{ij}(t) \mid 1 \leq i,j \leq N\}$. At the beginning of each episode, a patch corresponding to a quarter of the pixels is chosen at random in the grid. Pixels that do not belong to that patch evolve fully independently ($p(s_{ij}(t) = 1 \mid S(t-1)) = p(s_{ij}(t) = 1) = 0.5$). Pixels from the patch obey a local dependence law, in the form of a standard Ising model\footnote{\url{https://en.wikipedia.org/wiki/Ising_model}}: the value of a pixel at time $t$ only depends on the value of its neighbors at time $t-1$. This local dependence is obtained through a function $f$: $p(s_{ij}(t)|S(t-1)) = f(\{s_{i'j'}(t-1) \mid |i - i'| = |j - j'| = 1\})$ (see Appx~\ref{sec:appendix_ising_model} for details). Figure~\ref{fig:NCE_ising} shows the system at $t = 32$ during three different episodes (black pixels correspond to values of $-1$, white to $1$). The patches are very distinct from the noise. We then train a convolutional encoder using our DIM objective on local ``views'' only (see Section~\ref{sec:algo}). 

\begin{figure}[h]
    \centering
    \includegraphics[width=\linewidth]{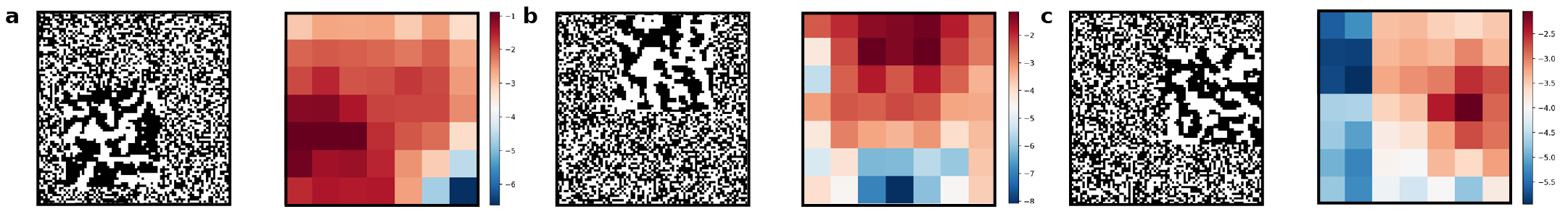}
    \caption{$42\times 42$ Ising model with temperature $\beta^{-1}=0.4$ overlaid onto a $84\times 84$ lattice of uniformly random spins $\{-1,+1\}$. The grayscale plots show each of the 3 systems at $t=32$; the color plots show the DIM similarity scores between $t=2$ and $t=3$. }
    \label{fig:NCE_ising}
\end{figure}

Figure~\ref{fig:NCE_ising} shows the similarity scores between the local features of states at $t=2$ and the same features at $t=3$ (a local feature corresponds to a specific location in the convolutional maps)\footnote{We chose early timesteps to make sure that the model does not simply detect large patches, but truly measures predictability.}. The heatmap regions containing the Ising model (larger-scale patterns) have higher scores than the noisy portions of the lattice. Local DIM is able to correctly encode regions of high temporal predictability.% because the Ising model determines the spin of the current node based on spins of its neighboring nodes.

\subsection{A continual learning experiment on Ms.~PacMan}
\label{sec:ms_pacman}
We further complicate the task of the Ising model prediction by building on top of the Ms.~PacMan game and introducing non-trivial dynamics. The environment is shown in the appendix. 

In order to assess how well our auxiliary objective captures predictability in this MDP, we define its dynamics such that $\mathbb{P}[\text{Ghost}_i\text{ takes a random move}]=\varepsilon$. Intuitively, as $\varepsilon\to 1$, the enemies' actions become less predictable, which in turn hinders the convergence rate of the contrastive loss. The four runs in Figure~\ref{fig:NCE_pacman}a correspond to various values of $\varepsilon$. We trained the agent using our $\cL_{DRIML}$ objective. We can see that the convergence of the auxiliary loss becomes slower with growing $\varepsilon$, as the model struggles to predict $s_{t+1}$ given $(s_t,a_t)$. After 100k frames, the NCE objective allows to separate the four MDPs according to their principal source of randomness (red and blue curves). When $\varepsilon$ is close to 1, the auxiliary loss has a harder time finding features that predict the next state, and eventually ignores the random movements of enemies.

The second and more interesting setup we consider consists in making only one out of 4 enemies lethal, and changing which one every 5k episodes. Figure~\ref{fig:NCE_pacman}b shows that, as training progresses, the blue curve (C51) always reaches the same performance at the end of the 5k episodes, while DRIML's steadily increases. The blue agent learns to ignore the harmless ghosts (they have no connection to the reward signal) and has to learn the task from scratch every time the lethal ghost changes. On the other hand, the DRIML agent (red curve) is incentivized to encode information about all the predictable objects on the screen (including the harmless ghosts), and as such adapts faster and faster to changes. Figure~\ref{fig:NCE_pacman}c shows the same PacMan environment with a quasi-deterministic Ising model evolving in the walled areas of the screen (details in appendix). For computational efficiency, we only run this experiment for 10k episodes. As before, DRIML outperforms C51 after the lethal ghost change, demonstrating that its representations encode more information about the dynamics of the environment (in particular about the harmless ghosts). The presence of additional distractors - the Ising model in the walls - did not impact that observation.

\begin{figure}[h!]
    \centering
    \includegraphics[width=\linewidth]{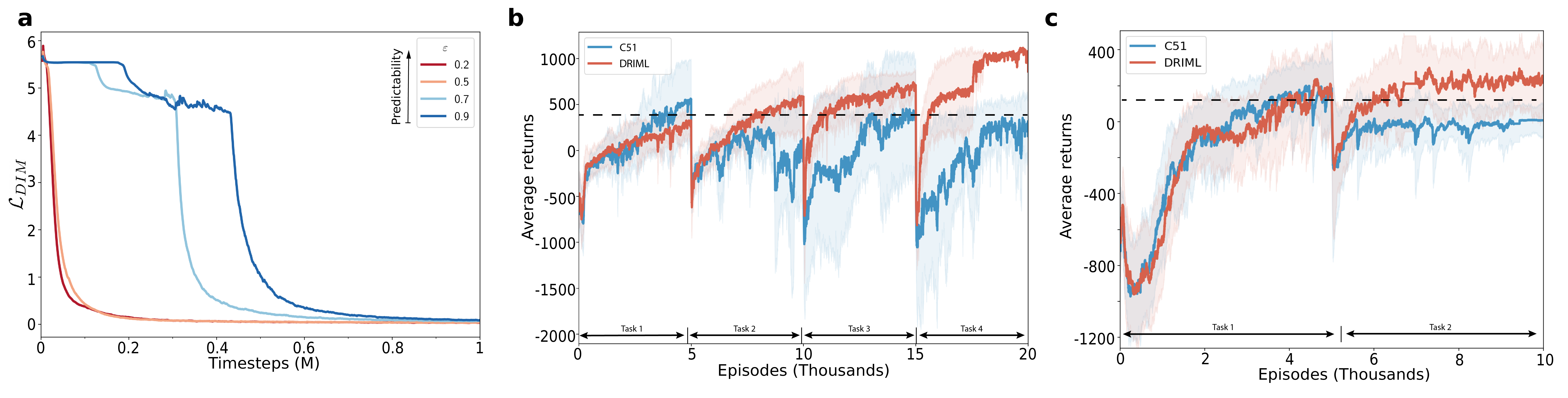}
    \caption{\textbf{(a)} average training NCE loss for various values of $\varepsilon$ as a function of timesteps, \textbf{(b)} average training reward with only one harmful enemy per level (dashed line indicates average terminal blue curve performance after each task) and \textbf{(c)} average training reward on PacMan + Ising noise in walled areas.}
    \label{fig:NCE_pacman}
\end{figure}

\subsection{Performance on Procgen Benchmark}
\label{sec:procgen_results}
Finally, we demonstrate the beneficial impact of adding a DIM-like objective to C51 (DRIML) on the 500 first levels of all 16 Procgen tasks~\citep{cobbe2019leveraging}. All algorithms are trained for 50M environment frames with the DQN~\citep{mnih2015human} architecture. The mean and standard deviation of the scores (over 3 seeds) are shown in Table~\ref{tab:procgen}; bold values indicate best performance. 

\begin{table}[ht]
\centering
\caption{Average training returns collected after 50M of training frames, $\pm$ one standard deviation.}
\resizebox{\linewidth}{!}{%
\begin{tabular}{l||l|ll|llll}
\toprule
Env & C51 & CPC-1$\rightarrow$ 5 & CURL & DRIML-noact & DRIML-randk & DRIML-fix & DRIML-ada\\ \midrule
bigfish    & $1.33\pm0.12$ & $1.17\pm0.16$ & $2.70\pm1.30$ & $1.19\pm0.04$ & $1.12\pm1.03$ & $2.02\pm0.18$ & \textbf{4.45} $\pm$ \textbf{0.71}\\
bossfight  & 0.57$\pm$0.05 & 0.52$\pm$0.07 & 0.60$\pm$0.06 & 0.47$\pm$0.01 & 0.56$\pm$0.03 & 0.67$\pm$0.02 & \textbf{1.05}$\pm$\textbf{0.19}\\
caveflyer  & 9.19$\pm$0.29 & 6.40$\pm$0.56 & 6.94$\pm$0.25 & 8.26$\pm$0.26 & 7.92$\pm$0.15 & \textbf{10.2}$\pm$\textbf{0.41} & 6.77$\pm$0.04\\
chaser     & 0.22$\pm$0.04 & 0.21$\pm$0.02 & 0.35$\pm$0.04 & 0.23$\pm$0.02 & 0.26$\pm$0.01 & 0.29$\pm$0.02 & \textbf{0.38}$\pm$\textbf{0.04}\\
climber    & 1.68$\pm$0.10 & 1.71$\pm$0.11 & 1.75$\pm$0.09 & 1.57$\pm$0.01 & 2.21$\pm$0.48 & \textbf{2.26}$\pm$\textbf{0.05} & 2.20$\pm$0.08\\
coinrun    & \textbf{29.7}$\pm$\textbf{5.44} & 11.4$\pm$1.55 & 21.2$\pm$1.94 & 13.2$\pm$1.21 & 21.6$\pm$1.97 & 27.2$\pm$1.92 & 22.88$\pm$0.4\\
dodgeball  & 1.20$\pm$0.08 & 1.05$\pm$0.04 & 1.09$\pm$0.04 & 1.22$\pm$0.04 & 1.19$\pm$0.03 & 1.28$\pm$0.02 & \textbf{1.44}$\pm$\textbf{0.06}\\
fruitbot   & 3.86$\pm$0.96 & 4.56$\pm$0.93 & 4.89$\pm$0.71 & 5.42$\pm$1.33 & 6.84$\pm$0.24 & 5.40$\pm$1.02 & \textbf{9.53}$\pm$\textbf{0.29}\\
heist      & 1.54$\pm$0.10 & 0.93$\pm$0.08 & 1.06$\pm$0.05 & 1.04$\pm$0.02 & 1.00$\pm$0.05 & 1.30$\pm$0.05 & \textbf{1.89}$\pm$\textbf{0.02}\\
jumper     & \textbf{13.2$\pm$0.83} & 2.28$\pm$0.44 & 10.3$\pm$0.61 & 4.31$\pm$0.64 & 5.62$\pm$0.27 & 12.6$\pm$0.64 & 12.2$\pm$0.42\\
leaper     & 5.03$\pm$0.14 & 4.01$\pm$0.71 & 3.94$\pm$0.46 & 5.40$\pm$0.09 & 4.24$\pm$1.17 & 6.17$\pm$0.29 & \textbf{6.35}$\pm$\textbf{0.46}\\
maze       & 2.36$\pm$0.09 & 1.14$\pm$0.08 & 0.82$\pm$0.20 & 1.44$\pm$0.26 & 1.18$\pm$0.03 & 1.38$\pm$0.08 & \textbf{2.62}$\pm$\textbf{0.10}\\
miner      & 0.13$\pm$0.01 & 0.13$\pm$0.02 & 0.10$\pm$0.01 & 0.12$\pm$0.01 & 0.15$\pm$0.01 & 0.14$\pm$0.01 & \textbf{0.19}$\pm$\textbf{0.02}\\
ninja      & \textbf{9.36}$\pm$\textbf{0.01} & 6.23$\pm$0.82 & 5.84$\pm$1.21 & 6.44$\pm$0.22 & 8.13$\pm$0.26 & 9.21$\pm$0.25 & 8.74$\pm$0.28\\
plunder    & 2.99$\pm$0.07 & 3.00$\pm$0.06 & 2.77$\pm$0.14 & 3.20$\pm$0.05 & 3.34$\pm$0.09 & 3.37$\pm$0.17 & \textbf{3.58}$\pm$\textbf{0.04}\\
starpilot  & 2.44$\pm$0.12 & 2.87$\pm$0.05 & 2.68$\pm$0.09 & 3.70$\pm$0.30 & 3.93$\pm$0.04 & \textbf{4.56}$\pm$\textbf{0.21} & 2.63$\pm$0.16\\ \midrule
Norm.score & $1.0$ & $0.23$ & $0.52$ & $0.59$ & $0.92$ & $1.48$ & $1.9$ \\ \bottomrule
\end{tabular}%
}
\label{tab:procgen}
\end{table}%
Similarly to CURL, we used data augmentation on inputs to DRIML-fix to improve the model's predictive capabilities in fast-paced environments (see App.~\ref{sec:appendix_procgen}). While we used the global-global loss in DRIML's objective for all Procgen games, we have found that the local-local loss also had a beneficial effect on performance on a smaller set of games (e.g. starpilot, which has few moving entities on a dark background).

%% file: discussion.tex
In this paper, we introduced an auxiliary objective called Deep Reinforcement and InfoMax Learning (DRIML), which is based on maximizing concordance of state-action pairs with future states (at the representation level). We presented results showing that 1) DRIML implicitly learns a transition model by boosting state similarity, 2) it can improve performance of deep RL agents in a continual learning setting and 3) it boosts training performance in complex domains such as Procgen.

%% file: appendix.tex
\subsection{Markov Chains}\label{sub:mc}
Given a discrete state space $\cS$ with probability measure $\mathbb{T}$, a discrete-time homogeneous Markov chain (MC) is a collection of random variables with the following property on its \textit{transition matrix} $\mat{T} \in \mathbb{R}^{|\cS| \times |\cS|}$: $\mat{T}_{s s'}=\mathbb{P}[S_{t+1} = s'|S_{t}=s],\;\forall s,s'\in \cS, \forall t \geq 0$. Assuming the Markov chain is ergodic, its \textit{invariant distribution}\footnote{The existence and uniqueness of $\vec{\rho}$ are direct results of the Perron-Frobenius theorem.} $\vec{\rho}$ is the principal eigenvector of $\mat{T}$, which verifies $\vec{\rho} \mat{T}=\vec{\rho}$ and summarizes the long-term behaviour of the chain. We define the marginal distribution of $S_t$ as $p_t(s):=\mathbb{P}[S_t=s]=\mat{T}^t_{s:}$,
and the initial distribution of $S$ as $p_0(s)$.

\subsection{Markov Decision Processes}

A discrete-time, finite-horizon Markov Decision Process~\citep[MDP]{bellman1957markovian,puterman2014markov}
comprises a state space $\cS$, an action space\footnote{We consider discrete state and action spaces.} $\cA$, a transition kernel $T:\cS\times \cA \times \cS\mapsto [0,1]$, a reward function $r:\cS\times \cA\mapsto \Real$ and a discount factor $\gamma\in[0,1]$.
At every timestep $t$, an agent interacting with this MDP observes the current state $s_t\in\cS$, selects an action $a_t\in\cA$, and observes a reward $r(s_t, a_t) \in\Real$ upon transitioning to a new state $s_{t+1}\sim T(s_t, a_t,\cdot)$.
The goal of an agent in a discounted MDP is to learn a policy $\pi:\cS\times\cA\mapsto[0,1]$ such that taking actions $a_t\sim\pi(\cdot|s_t)$ maximizes the expected sum of discounted returns,
\begin{align*}
V^\pi(s) = \mathbb{E}_\pi \bigg[\sum_{t=0}^\infty \gamma^t r(s_t, a_t) | s_0 = s\bigg].
\end{align*}

To convert a MDP into a MC, one can let $\mat{T}^\pi_{ss'}=\mathbb{E}_{a \sim \pi(s,\cdot)}[T(s, a, s')]$, an operation which can be easily tensorized for computational efficiency in small state spaces~\citep[see][]{mazoure2020provably}.

\subsection{Link to invariant distribution}
For a discrete state ergodic Markov chain specified by $\mat{P}$ and initial occupancy vector $\vec{p}_0$, its marginal state distribution at time $t$ is given by the Chapman-Kolmogorov form:
\begin{equation}
    \mathbb{P}[S_t=s]=\vec{p}_0\mat{P}^t_{\cdot s},
\end{equation}
and its \textit{limiting distribution} $\vec{\sigma}$ is the infinite-time marginal
\begin{equation}
    \lim_{t\to \infty} \mat{P}^{(t)}_{ss'}=\vec{\sigma}_{s'}, \; s,s' \in \cS
\end{equation}
which, if it exists, is exactly equal to the \textit{invariant distribution} $\vec{\rho}$.

For the very restricted family of ergodic MDPs under fixed policy $\pi$, we can assume that $p_t$ converges to a time invariant distribution $\rho$.

Therefore,
\begin{align}
\begin{split}
    \mathcal{I}_t(S,S')&=\sum_{s'\in \cS}\sum_{s\in\cS}p_0(\mat{P}^t)_{: s}\mat{P}_{ss'}\bigg(\log \mat{P}_{ss'} - \log\{p_0(\mat{P}^{t+1})_{: s'}\}\bigg)\\
\end{split}
\end{align}
Now, observe that $\cI_t$ is closely linked to $T/\rho$ when samples come from timesteps close to $t_{mix}(\varepsilon)$. That is, interchanging swapping $\rho(s)$ and $p_t(s)$ at any state $s$ would yield at most $\delta(t)$ error. Moreover, existing results~\citep{levin2017markov} from Markov chain theory provide bounds on $||(\mat{P}^{t+1})_{s:}-(\mat{P}^t)_{s:}||_{TV}$ depending on the structure of the transition matrix.

If $\mat{P}$ has a limiting distribution $\vec{\sigma}$, then using the dominated convergence theorem allows to replace matrix powers by $\vec{\sigma}$, which is then replaced by the invariant distribution $\vec{\rho}$:
\begin{equation}
\begin{split}
     \lim_{t\to \infty}\cI_t &= \sum_{s'\in \cS}\sum_{s\in\cS}\vec{\rho}_{s}\vec{\rho}_{s'}\bigg(\log \mat{P}_{ss'} - \log\vec{\rho}_{s'}\bigg)\\
     &= \mathbb{E}_{\rho \times \rho}\bigg(\log \mat{P}_{ss'} -\log\vec{\rho}_{ s'} \bigg)
     \label{eq:limiting_pmi}
\end{split}
\end{equation}

Of course, most real-life Markov decision processes do not actually have an invariant distribution since they have absorbing (or terminal) states. In this case, as the agent interacts with the environment, the DIM estimate of MI yields a rate of convergence which can be estimated based on the spectrum of $\mat{P}$.

Moreover, one could argue that since, in practice, we use off-policy algorithms for this sort of task, the gradient signal comes from various timesteps within the experience replay, which drives the model to learn features that are consistently predictive through time.

\subsection{Predictability and Contrastive Learning}
\input{theory}

\subsection{Code snippet for DIM objective scores}
The following snippet yields pointwise (i.e. not contracted) scores given a batch of data.
\label{code:dim_batch_scores}
\input{DIM_snippet}

To obtain a scalar out of this batch, sum over the third dimension and then average over the first two.

\subsection{Experiment details}
\label{sec:appendix_experiment_details}
All experiments involving RGB inputs (Ising, Ms.PacMan and Procgen) were ran with the settings shown in Table~\ref{tab:appendix_experiment_params}. Parameters such as gradient clipping and n-step-returns were kept from the codebase, `rlpyt`, since it was observed that they helped achieve a more stable convergence.
\begin{table}[h]
    \centering
    \begin{tabular}{c|l|c}
        Name & Description & Value \\
        \hline 
       $\varepsilon_{T_{exploration}}$  & Exploration at $t=0$ & 0.1\\
       $\varepsilon_{T_{exploration}}$ & Exploration at $t=T_{exploration}$ & 0.01\\
       $T_{exploration}$ & Exploration decay & $10^5$\\
       LR & Learning rate & $2.5\times 10^{-4}$\\
       $\gamma$ & Discount factor & $0.99$\\
       Clip grad & Gradient clip norm & $10$\\
       N-step-return & N-step return & $7$\\
       \multirow{2}{*}{Frame stack} & \multirow{2}{*}{Number of stacked frames} & $1$ (Ising and Procgen)\\
       & & $4$ (Ms.PacMan)\\
       Grayscale & Grayscale or RGB & RGB\\
       \multirow{2}{*}{Input size} & \multirow{2}{*}{State input size} & $84\times 84$ (Ising and Ms.PacMan)\\
       & & $80 \times 104$ (Procgen)\\
       $T_{warmup}$ & Warmup steps & 1000\\
       Replay size & Size of replay buffer & $10^6$\\
       $\tau$ & Target soft update coeff & $0.95$\\
       Clip reward & Reward clipping & False\\
       \multirow{2}{*}{$\lambda_{4t4}$} & \multirow{2}{*}{Global-global DIM} & 1 (Ms.PacMan and Procgen)\\
       & & 0 (Ising)\\
       \multirow{2}{*}{$\lambda_{3t3}$} & \multirow{2}{*}{Local-local DIM} & 1 (Ms.PacMan and Ising)\\
       & & 0 (Procgen)\\
       $\lambda_{3t4}$ & Local-global DIM & 0\\
       $\lambda_{4t3}$ & Global-local DIM & 0\\
        \multirow{2}{*}{$k$} &  \multirow{2}{*}{DIM lookahead constant} & 1 (Ising and Ms.PacMan)\\
       & & Variable between 1 and 5 (Procgen)
    \end{tabular}
    \caption{Experiments' parameters}
    \label{tab:appendix_experiment_params}
\end{table}

The global DIM heads consist of a standard single hidden fully-connected layer network of 512 with ReLU activations and a skip-connection from input to output layers. The action is transformed into one-hot and then encoded using a 64 unit layer, after which it is concatenated with the state and passed to the global DIM head.

The local DIM heads consist of a single hidden layer network made of $1\times 1$ convolution. The action is tiled to match the shape of the convolutions, encoded using a $1\times 1$ convolutions and concatenated along the feature dimension with the state, after which is is passed to the local DIM head.

In the case of the Ising model, there is no decision component and hence no concatenation of state and action representations is required.

\subsubsection{AMI of a biased random walk}
\label{sec:appendix_single_mc}
We see from the formulation of the mutual information objective $\mathcal{I}(S_{t+1},S_t)$ that it inherently depends on the ratio of $\mat{P}/\rho$. Recall that, for a 1-d random walk on integers $[0,N)$, the stationary distribution is a function of $\frac{\alpha}{1-\alpha}$ and can be found using the recursion $\mathbb{P}[S=i]=\alpha\mathbb{P}[S=i-1]+(1-\alpha)\mathbb{P}[S=i+1]$. It has the form
\begin{equation}
    \rho(i)=\mathbb{P}[S=i]=r^i(1-r)(1-r^{N})^{-1},\;\; i \in [0,N),
\end{equation}
for $r=\frac{\alpha}{1-\alpha}$.

The pointwise mutual information between states $S_t$ and $S_{t+1}$ is therefore the random variable
\begin{equation}
\begin{split}
    \dot{\cI}(S_{t+1},S_t)&=\log \mathbb{P}[S_{t+1}|S_t]-\log \mathbb{P}[S_{t+1}]\\
    &=\log \alpha ^{\mathbbm{1}_{(>0)}(S_{t+1}-S_t)}+\log (1-\alpha)^{\mathbbm{1}_{(<0)}(S_{t+1}-S_t)}-\log \rho(S_{t+1})
\end{split}
\end{equation}
with expectation equal to the average mutual information which we can find by maximizing, among others, the InfoNCE bound. We can then compute the AMI as a function of $\alpha$
\begin{equation}
    \cI(S_{t+1},S_t;\alpha)=\sum_{i=0}^N\sum_{j=0}^N\dot{\cI}(j,i) \alpha ^{\mathbbm{1}_{(>0)}(j-i)} (1-\alpha)^{\mathbbm{1}_{(<0)}(j-i)} \rho(i),
\end{equation}
which is shown in Figure~\ref{fig:fig1_random_walk_p}c.

The figures were obtained by training the global DIM objective $\Phi_4$ on samples from the chain for 1,000 epochs with learning rate $10^{-3}$.

\subsubsection{Ising model}
\label{sec:appendix_ising_model}
 We start by generating an $84 \times 84$ rectangular lattice which is filled with Rademacher random variables $v_{1,1},..,v_{84,84}$; that is, taking $-1$ or $1$ with some probability $p$. For any $p\in (0,1)$, the joint distribution $p(v_{1,1},..,v_{84,84})$ factors into the product of marginals $p(v_{1,1})..p(v_{84,84})$.

At every timestep, we uniformly sample a random index tuple $(i,j),21\leq i,j\leq 63$ and evolve the set of nodes $\vec{v}=\{v_{k,l}:i-21\leq k \leq i+21,j-21\leq l \leq j+21\}$ according to an Ising model with temperature $\beta^{-1}=0.4$, while the remaining nodes continue to independently take the values $\{-1,1\}$ with equal probability. If one examines any subset of nodes outside of $\vec{v}$, then the information conserved across timesteps would be close to 0, due to observations being independent in time.

However, examining a subset of $\vec{v}$ at timestep $t$ allows models based on mutual information maximization to predict the configuration of the system at $t+1$, since this region has high mutual information across time due to the ratio $\frac{\mat{T}(v,v')}{p_{t+1}(v')}$ being directly proportional to the temperature parameter $\beta^{-1}$.

To obtain the figure, we trained local DIM $\Phi_3$ on sample snapshots of the Ising model as $84\times84$ grayscale images for 10 epochs. The local DIM scores were obtained by feeding a snapshot of the Ising model at $t=3$; showing it snapshots from later timestep would've made the task much easier since there would be a clear difference in granularities of the random pattern and Ising models.

\subsubsection{Ms.PacMan}
\begin{figure}[h!]
    \centering
    \includegraphics[width=0.4\linewidth]{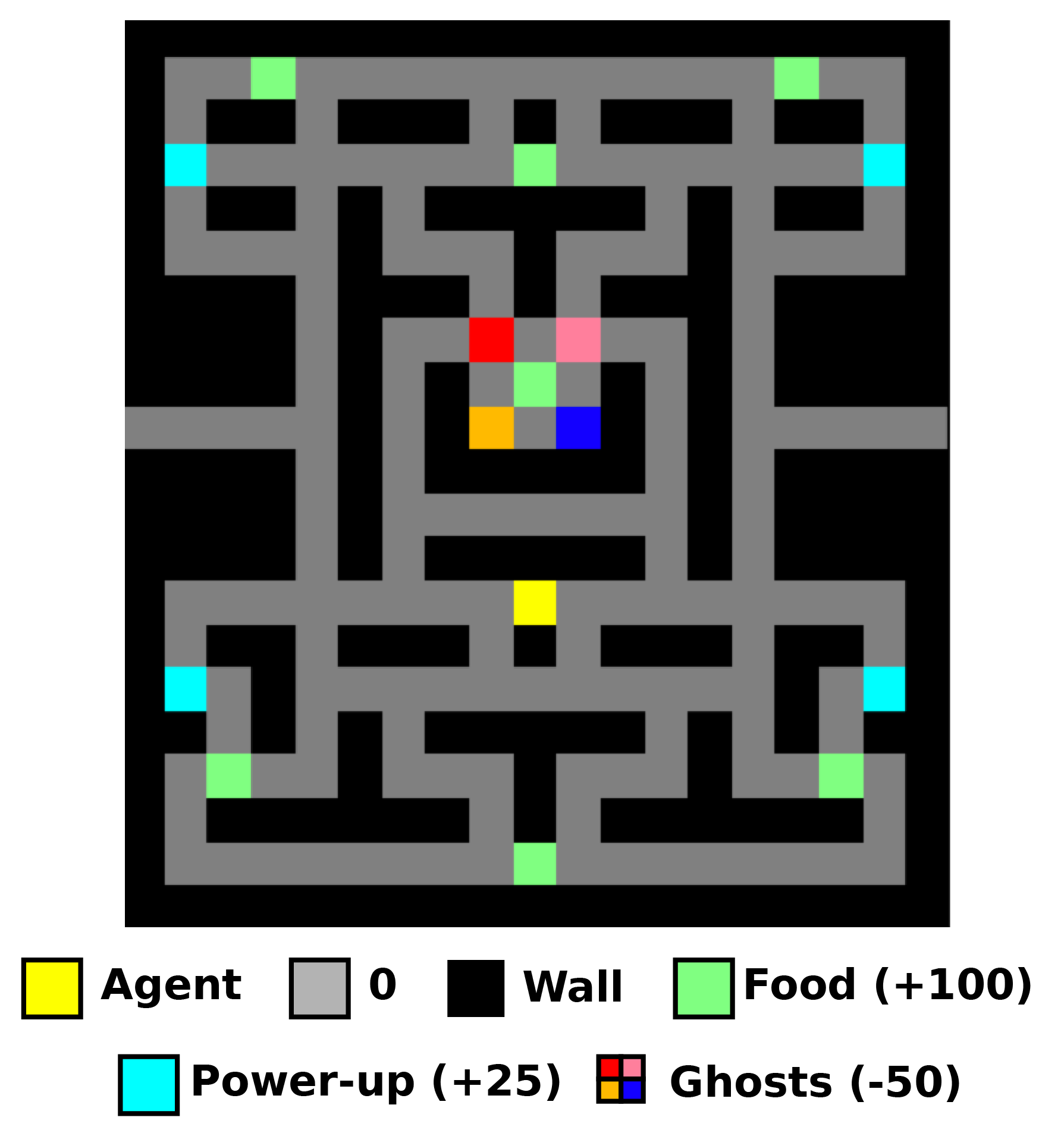}
    \caption{The simplified Ms.PacMan environment}
    \label{fig:pacman_env}
\end{figure}
In PacMan, the agent, represented by a yellow square, must collect food pellets while avoiding four harmful ghosts. When the agent collects one of the boosts, it becomes invincible for 10 steps, allowing it to destroy the enemies without dying. In their turn, ghosts alternate between three behaviours: 1) when the agent is not within line-of-sight, wander randomly, 2) when the agent is visible and does not have a boost, follow them and 3) when the agent is visible and has a boost, avoid them. The switch between these three modes happens stochastically and quasi-independently for all four ghosts. Since the food and boost pellets are fixed at the beginning of each episode, randomness in the MDP comes from the ghosts as well as the agent's actions. 

The setup for our first experiment in the domain is as follows: with a fixed probability $\varepsilon$, each of the 4 enemies take a random action instead of following one of the three movement patterns.

The setup for our second experiment in the domain consists of four levels: in each level, only one out of the four ghosts is lethal - the remaining three behave the same but do not cause damage. The model trains for 5,000 episodes on level 1, then switches to level 3, then level 3 and so forth.
This specific environment tests for the ability of DIM to quickly figure out which of the four enemies is the lethal one and ignore the remaining three based on color .

For our study, the state space consisted of $21 \times 19 \times 3$ RGB images. The inputs to the model were states re-scaled to $42 \times 38 \times 12$ by stacking 4 consecutive frames, which were then concatenated with actions using an embedding layer.

The third experiment consisted in overlaying the Ising model from the above section onto walls in the Ms.PacMan game. Every rollout, the Ising model was reset to some (random) initial configuration and allowed to evolve until termination of the episode. The color of the Ising distractor features was chosen to be fuchsia.

\begin{figure}[h!]
    \centering
    \includegraphics[width=0.5\linewidth]{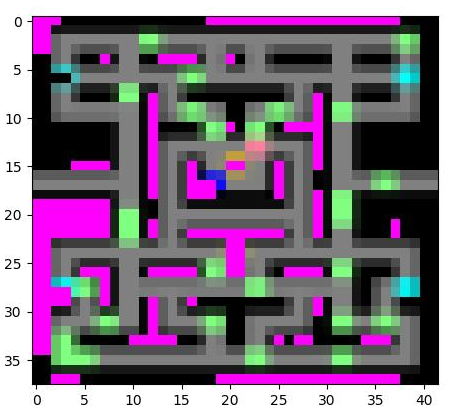}
    \caption{Upsampled screen cap of the Ms.PacMan task with Ising distractor features (in fuchsia).}
    \label{fig:pacman_ising}
\end{figure}

\subsubsection{Procgen}
\label{sec:appendix_procgen}

The training setting consists in fixing the first 500 levels of a given Procgen game, and train all algorithms on these 500 levels in that specific order. Since we use the Nature architecture of DQN rather than IMPALA (due to computational restrictions), our results can be different from other Procgen baselines.

The data augmentation was tried only for DRIML-fix - DRIML-ada seems to perform well without data augmentation. The data augmentation steps performed on $S_t$ and $S_{t+k}$ fed to the DIM loss consisted of a random crop ($0.8$ of the original's size) with color jitter with parameters $0.4$. Although the data augmentation is helpful on some tasks (typically fast-paced, requiring a lot of camera movements), it has shown detrimental effects on others. Below is a list of games on which data augmentation was beneficial: bigfish, bossfight, chaser, coinrun, jumper, leaper and ninja.

The $k$ parameter, which specifies how far into the future the model should make its predictions, worked best when set to 5 on the games: bigfish, chaser, climber, fruitbot, jumper, miner, maze and plunder. For the remaining games, setting $k=1$ yielded better performance.

\paragraph{Baselines} The baselines were implemented on top of our existing architecture and, for models which use contrastive objectives, used the exactly same networks for measuring similarity (i.e. one residual block for CURL and CPC). CURL was implemented based on the authors' code included in their paper and that of MoCo, with EMA on the target network as well as data augmentation (random crops and color jittering) on $S_t$ for randomly sampled $t>0$.

The No Action baseline was tuned on the same budget as DRIML, over $k=1,5$ and with/without data augmentation. Best results are reported in the main paper.
\begin{table}[ht]
\centering
\begin{tabular}[t]{lcc}
\hline
 & DRIML-noact (k=1) & DRIML-noact (k=5)\\
\hline
bigfish & 1.193 $\pm$ 0.04 & 1.33 $\pm$ 0.12 \\
bossfight & 0.466 $\pm$ 0.07 & 0.472 $\pm$ 0.01 \\
caveflyer & 8.263 $\pm$ 0.26 & 5.925 $\pm$ 0.18 \\
chaser & 0.224 $\pm$ 0.01 & 0.229 $\pm$ 0.02 \\
climber & 1.359 $\pm$ 0.13 & 1.574 $\pm$ 0.01 \\
coinrun & 13.146 $\pm$ 1.21 & 9.632 $\pm$ 2.8 \\
dodgeball & 1.221 $\pm$ 0.04 & 1.213 $\pm$ 0.09 \\
fruitbot & 0.714 $\pm$ 0.31 & 5.425 $\pm$ 1.33 \\
heist & 1.042 $\pm$ 0.02 & 0.861 $\pm$ 0.07 \\
jumper & 2.966 $\pm$ 0.1 & 4.314 $\pm$ 0.64 \\
leaper & 5.403 $\pm$ 0.09 & 3.521 $\pm$ 0.3 \\
maze & 0.984 $\pm$ 0.13 & 1.438 $\pm$ 0.26 \\
miner & 0.11 $\pm$ 0.01 & 0.116 $\pm$ 0.01 \\
ninja & 6.437 $\pm$ 0.22 & 5.9 $\pm$ 0.38 \\
plunder & 2.67 $\pm$ 0.08 & 3.2 $\pm$ 0.05 \\
starpilot & 3.699 $\pm$ 0.3 & 2.951 $\pm$ 0.31 \\
\hline
\end{tabular}
\caption{Ablation of the impact of predictive timestep in NCE objective (i.e. $k$) on the no action model's training performance (50M training frames).}
\label{tab:procgen_ablation_nstep_no_action}
\end{table}%

%% file: theory.tex
\label{sec:predictability_nce}
Information maximization has long been considered one of the standard principles for measuring correlation and performing feature selection~\citep{song2012comparison}. In the MDP context, high values of $\mathcal{I}([S_t,A_t],S_{t+k})$ indicate that $(S_t,A_t)$ and $S_{t+k}$ have some form of dependence, while low values suggest independence. The fact that predictability (or more precisely determinism) in Markov systems is linked to 
the MI suggests a deeper connection to the spectrum of the transition kernel $T$. For instance, the set of eigenvalues of $T$ for a Markov decision process contains important information about the connectivity of said process, such as mixing time or number of densely connected clusters~\citep{von2007tutorial,levin2017markov}. 

Consider the setting in which $\pi$ is fixed at some iteration in the optimization process. In the rest of this section, we let $\mat{T}$ denote the expected transition model $\mat{T}(s,s')=\mathbb{E}_\pi[T(s,a,s')]$ (it is a Markov chain). We let $\nu_t(s,s') = \frac{\mat{T}(s,s')}{p_{t+1}(s')}$ be the ratio learnt when optimizing the infoNCE loss on samples drawn from the random variables $S_t$ and $S_{t+1}$ (for a fixed $t$)~\citep{oord2018representation}. We also let $\nu_\infty(s,s') = \frac{\mat{T}(s,s')}{\vec{\rho}(s')}$ be that ratio when the Markov chain has reached its stationary distribution $\vec{\rho}$ (see Section~\ref{sub:mc}), and $\tilde{\nu}_t(s,s')$ be the scoring function learnt using InfoNCE (which converges to $\nu_t(s,s')$ in the limit of infinite samples drawn from $(S_t,S_{t+1})$).

\begin{proposition}

Let $0 < \epsilon \leq 1$. Assume at time step $t$, training of $\tilde{\nu}_t$ has close to converged on a pair $(s,s')$, i.e. $|\nu_t(s,s') - \tilde{\nu}_t(s,s')| < \epsilon$. Then the following holds:

\begin{equation}
    t \geq t_{mix}\bigg(\frac{\epsilon}{2} \min_{x}\vec{\rho}(x)^2\bigg) \quad \implies \quad \bigg|\nu_t(s,s') - \nu_{\infty}(s,s')\bigg| \leq 2 \epsilon.
\end{equation}

\label{prop:nu_ratio}
\end{proposition}

\begin{proof}
Let us consider fixed $0 < \epsilon \leq 1$, $(s,s')$ and $t \geq t_{mix}(\frac{\epsilon}{2} \min_{x}\vec{\rho}(x)^2)$.
First, since $t_{mix}(\frac{\epsilon}{2} \min_{x}\vec{\rho}(x)^2) \geq t_{mix}(\frac{\min_{x}\vec{\rho}(x)}{2})$, we have
\begin{equation*}
    \quad |p^*_{t+1}(s') - \vec{\rho}(s')| \leq \frac{\min_{x}\vec{\rho}(x)}{2}.
\end{equation*}
Or in other terms: $p^*_{t+1}(s') \geq \frac{\min_{x}\vec{\rho}(x)}{2}$.
Now, we have:
\begin{align*}
    |\tilde{\nu}_t(s,s') - \nu_{\infty}(s,s')| &\leq |\tilde{\nu}_t(s,s') - \nu_t(s,s')| + |\nu_t(s,s') - \nu_{\infty}(s,s')| \\
    &\leq \epsilon + \bigg|\frac{\mat{T}(s,s')}{p_{t+1}(s')} - \frac{\mat{T}(s,s')}{\vec{\rho}(s')}\bigg| \\
    &\leq \epsilon + \frac{|p_{t+1}(s') - \vec{\rho}(s')|}{p_{t+1}(s') \vec{\rho}(s')} \\
    &\leq \epsilon + \frac{|p_{t+1}(s') - \vec{\rho}(s')|}{\frac{\min_{x}\vec{\rho}(x)}{2} \min_{x}\vec{\rho}(x)}.
\end{align*}
By assumption on $t$, we know that $|p_{t+1}(s') - \vec{\rho}(s')| \leq \frac{\epsilon}{2} \min_{x}\vec{\rho}(x)^2$, which concludes the proof.
\end{proof}

Proposition~\ref{prop:nu_ratio} in conjunction with the result on Markov chain mixing times from \cite{levin2017markov} suggests that faster convergence of $\tilde{\nu}_t$ to $\nu_\infty$ happens when the \textit{spectral gap} $1-\lambda_{(2)}$ of $T$ is large, or equivalently when $\lambda_{(2)}$ is small. It follows that, on one hand, mutual information is a natural measure of concordance of $(s,s')$ pairs and can be maximized using data-efficient, batched gradient methods. On the other hand, the rate at which the InfoNCE loss converges to its stationary value (ie maximizes the lower bound on MI) depends on the spectral gap of $T$, which is closely linked to predictability. This relation holds in very simple domains like the Markov Chains which were presented across the paper, but for now, there is no reliable way to estimate the second eigenvalue of the MDP transition operator under nonlinear function approximation that we are aware of.

%% file: DIM_snippet.tex
\begin{minted}[mathescape,
               linenos,
               numbersep=5pt,
               gobble=2,
               frame=lines,
               framesep=2mm]{python}
               
    def temporal_DIM_scores(reference,positive,clip_val=20):
        """
        reference: n_batch × n_rkhs × n_locs
        positive: n_batch x n_rkhs x n_locs
        """
        reference = reference.permute(2,0,1)
        positive = positive.permute(2,1,0)
        # reference: n_loc × n_batch × n_rkhs
        # positive: n_locs × n_rkhs × n_batch
        pairs = torch.matmul(reference, positive) 
        # pairs: n_locs × n_batch × n_batch
        pairs = pairs / reference.shape[2]**0.5
        pairs = clip_val * torch.tanh((1. / clip_val) * pairs)
        shape = pairs.shape
        scores = F.log_softmax(pairs, 2)
        # scores: n_locs × n_batch × n_batch
        mask = torch.eye(shape[2]).unsqueeze(0).repeat(shape[0],1,1)
        # mask: n_locs × n_batch × n_batch
        scores = scores * mask
        # scores: n_locs × n_batch × n_batch
        return scores
\end{minted}

%% file: main.bbl
\begin{thebibliography}{51}
\providecommand{\natexlab}[1]{#1}
\providecommand{\url}[1]{\texttt{#1}}
\expandafter\ifx\csname urlstyle\endcsname\relax
  \providecommand{\doi}[1]{doi: #1}\else
  \providecommand{\doi}{doi: \begingroup \urlstyle{rm}\Url}\fi

\bibitem[Sutton and Barto(1998)]{SuttonBarto98}
R.~S. Sutton and A.~G. Barto.
\newblock \emph{Reinforcement Learning: An Introduction.}
\newblock MIT Press, Cambridge, MA, 1998.

\bibitem[Ha and Schmidhuber(2018)]{ha2018world}
David Ha and Jürgen Schmidhuber.
\newblock World models.
\newblock 2018.
\newblock \doi{10.5281/zenodo.1207631}.
\newblock URL \url{http://arxiv.org/abs/1803.10122}.
\newblock cite arxiv:1803.10122.

\bibitem[Hafner et~al.(2019{\natexlab{a}})Hafner, Lillicrap, Ba, and
  Norouzi]{hafner2019dream}
Danijar Hafner, Timothy Lillicrap, Jimmy Ba, and Mohammad Norouzi.
\newblock Dream to control: Learning behaviors by latent imagination.
\newblock \emph{arXiv preprint arXiv:1912.01603}, 2019{\natexlab{a}}.

\bibitem[Goodfellow et~al.(2017)Goodfellow, Bengio, and
  Courville]{goodfellow2017learning}
Ian Goodfellow, Yoshua Bengio, and Aaron Courville.
\newblock \emph{Deep learning}.
\newblock 2017.
\newblock ISBN 9780262035613 0262035618.
\newblock URL
  \url{https://www.worldcat.org/title/deep-learning/oclc/985397543&referer=brief_results}.

\bibitem[Mnih et~al.(2015)Mnih, Kavukcuoglu, Silver, Rusu, Veness, Bellemare,
  Graves, Riedmiller, Fidjeland, Ostrovski, et~al.]{mnih2015human}
Volodymyr Mnih, Koray Kavukcuoglu, David Silver, Andrei~A Rusu, Joel Veness,
  Marc~G Bellemare, Alex Graves, Martin Riedmiller, Andreas~K Fidjeland, Georg
  Ostrovski, et~al.
\newblock Human-level control through deep reinforcement learning.
\newblock \emph{Nature}, 518\penalty0 (7540):\penalty0 529--533, 2015.

\bibitem[Schulman et~al.(2017)Schulman, Wolski, Dhariwal, Radford, and
  Klimov]{schulman2017ppo}
John Schulman, Filip Wolski, Prafulla Dhariwal, Alec Radford, and Oleg Klimov.
\newblock Proximal policy optimization algorithms.
\newblock \emph{CoRR}, abs/1707.06347, 2017.
\newblock URL
  \url{http://dblp.uni-trier.de/db/journals/corr/corr1707.html#SchulmanWDRK17}.

\bibitem[Pong et~al.(2018)Pong, Gu, Dalal, and Levine]{pong2018temporal}
Vitchyr Pong, Shixiang Gu, Murtaza Dalal, and Sergey Levine.
\newblock Temporal difference models: Model-free deep rl for model-based
  control.
\newblock \emph{arXiv preprint arXiv:1802.09081}, 2018.

\bibitem[Farebrother et~al.(2018)Farebrother, Machado, and
  Bowling]{farebrother2018generalization}
Jesse Farebrother, Marlos~C. Machado, and Michael Bowling.
\newblock Generalization and regularization in dqn, 2018.

\bibitem[Tachet~des Combes et~al.(2018)Tachet~des Combes, Bachman, and van
  Seijen]{journals/corr/abs-1809-02591}
Remi Tachet~des Combes, Philip Bachman, and Harm van Seijen.
\newblock Learning invariances for policy generalization.
\newblock \emph{CoRR}, abs/1809.02591, 2018.
\newblock URL
  \url{http://dblp.uni-trier.de/db/journals/corr/corr1809.html#abs-1809-02591}.

\bibitem[Hjelm et~al.(2018)Hjelm, Fedorov, Lavoie-Marchildon, Grewal, Bachman,
  Trischler, and Bengio]{hjelm2018learning}
R~Devon Hjelm, Alex Fedorov, Samuel Lavoie-Marchildon, Karan Grewal, Phil
  Bachman, Adam Trischler, and Yoshua Bengio.
\newblock Learning deep representations by mutual information estimation and
  maximization.
\newblock \emph{arXiv preprint arXiv:1808.06670}, 2018.

\bibitem[Bachman et~al.(2019)Bachman, Hjelm, and
  Buchwalter]{bachman2019learning}
Philip Bachman, R~Devon Hjelm, and William Buchwalter.
\newblock Learning representations by maximizing mutual information across
  views.
\newblock In \emph{Advances in Neural Information Processing Systems}, pages
  15509--15519, 2019.

\bibitem[Anand et~al.(2019)Anand, Racah, Ozair, Bengio, C{\^o}t{\'e}, and
  Hjelm]{anand2019unsupervised}
Ankesh Anand, Evan Racah, Sherjil Ozair, Yoshua Bengio, Marc-Alexandre
  C{\^o}t{\'e}, and R~Devon Hjelm.
\newblock Unsupervised state representation learning in atari.
\newblock In \emph{Advances in Neural Information Processing Systems}, pages
  8766--8779, 2019.

\bibitem[Hyvarinen and Morioka(2016)]{hyvarinen2016unsupervised}
Aapo Hyvarinen and Hiroshi Morioka.
\newblock Unsupervised feature extraction by time-contrastive learning and
  nonlinear ica.
\newblock In \emph{Advances in Neural Information Processing Systems}, pages
  3765--3773, 2016.

\bibitem[Bellemare et~al.(2013)Bellemare, Naddaf, Veness, and
  Bowling]{bellemare2013arcade}
Marc~G Bellemare, Yavar Naddaf, Joel Veness, and Michael Bowling.
\newblock The arcade learning environment: An evaluation platform for general
  agents.
\newblock \emph{Journal of Artificial Intelligence Research}, 47:\penalty0
  253--279, 2013.

\bibitem[Cobbe et~al.(2019)Cobbe, Hesse, Hilton, and
  Schulman]{cobbe2019leveraging}
Karl Cobbe, Christopher Hesse, Jacob Hilton, and John Schulman.
\newblock Leveraging procedural generation to benchmark reinforcement learning,
  2019.

\bibitem[Wixted(2004)]{wixted2004psychology}
John~T Wixted.
\newblock The psychology and neuroscience of forgetting.
\newblock \emph{Annu. Rev. Psychol.}, 55:\penalty0 235--269, 2004.

\bibitem[Atkinson et~al.(2018)Atkinson, McCane, Szymanski, and
  Robins]{atkinson2018pseudo}
Craig Atkinson, Brendan McCane, Lech Szymanski, and Anthony Robins.
\newblock Pseudo-rehearsal: Achieving deep reinforcement learning without
  catastrophic forgetting.
\newblock \emph{arXiv preprint arXiv:1812.02464}, 2018.

\bibitem[Kaplanis et~al.(2018)Kaplanis, Shanahan, and
  Clopath]{kaplanis2018continual}
Christos Kaplanis, Murray Shanahan, and Claudia Clopath.
\newblock Continual reinforcement learning with complex synapses.
\newblock \emph{arXiv preprint arXiv:1802.07239}, 2018.

\bibitem[Mankowitz et~al.(2018)Mankowitz, {\v{Z}}{\'\i}dek, Barreto, Horgan,
  Hessel, Quan, Oh, van Hasselt, Silver, and Schaul]{mankowitz2018unicorn}
Daniel~J Mankowitz, Augustin {\v{Z}}{\'\i}dek, Andr{\'e} Barreto, Dan Horgan,
  Matteo Hessel, John Quan, Junhyuk Oh, Hado van Hasselt, David Silver, and Tom
  Schaul.
\newblock Unicorn: Continual learning with a universal, off-policy agent.
\newblock \emph{arXiv preprint arXiv:1802.08294}, 2018.

\bibitem[Doan et~al.(2020)Doan, Bennani, Mazoure, Rabusseau, and
  Alquier]{doan2020theoretical}
Thang Doan, Mehdi Bennani, Bogdan Mazoure, Guillaume Rabusseau, and Pierre
  Alquier.
\newblock A theoretical analysis of catastrophic forgetting through the ntk
  overlap matrix.
\newblock \emph{arXiv preprint arXiv:2010.04003}, 2020.

\bibitem[Thrun and Pratt(1998)]{thrun1998learning}
Sebastian Thrun and Lorien Pratt.
\newblock Learning to learn: Introduction and overview.
\newblock In \emph{Learning to learn}, pages 3--17. Springer, 1998.

\bibitem[Finn et~al.(2017)Finn, Abbeel, and Levine]{finn2017model}
Chelsea Finn, Pieter Abbeel, and Sergey Levine.
\newblock Model-agnostic meta-learning for fast adaptation of deep networks.
\newblock In \emph{Proceedings of the 34th International Conference on Machine
  Learning-Volume 70}, pages 1126--1135. JMLR. org, 2017.

\bibitem[Hessel et~al.(2019)Hessel, Soyer, Espeholt, Czarnecki, Schmitt, and
  van Hasselt]{hessel2019multi}
Matteo Hessel, Hubert Soyer, Lasse Espeholt, Wojciech Czarnecki, Simon Schmitt,
  and Hado van Hasselt.
\newblock Multi-task deep reinforcement learning with popart.
\newblock In \emph{Proceedings of the AAAI Conference on Artificial
  Intelligence}, volume~33, pages 3796--3803, 2019.

\bibitem[D'Eramo et~al.(2019)D'Eramo, Tateo, Bonarini, Restelli, and
  Peters]{d2019sharing}
Carlo D'Eramo, Davide Tateo, Andrea Bonarini, Marcello Restelli, and Jan
  Peters.
\newblock Sharing knowledge in multi-task deep reinforcement learning.
\newblock In \emph{International Conference on Learning Representations}, 2019.

\bibitem[Jaderberg et~al.(2016)Jaderberg, Mnih, Czarnecki, Schaul, Leibo,
  Silver, and Kavukcuoglu]{jaderberg2016reinforcement}
Max Jaderberg, Volodymyr Mnih, Wojciech~Marian Czarnecki, Tom Schaul, Joel~Z
  Leibo, David Silver, and Koray Kavukcuoglu.
\newblock Reinforcement learning with unsupervised auxiliary tasks.
\newblock \emph{arXiv preprint arXiv:1611.05397}, 2016.

\bibitem[Mohamed and Rezende(2015)]{mohamed2015variational}
Shakir Mohamed and Danilo~Jimenez Rezende.
\newblock Variational information maximisation for intrinsically motivated
  reinforcement learning.
\newblock In \emph{Advances in neural information processing systems}, pages
  2125--2133, 2015.

\bibitem[Gelada et~al.(2019)Gelada, Kumar, Buckman, Nachum, and
  Bellemare]{gelada2019deepmdp}
Carles Gelada, Saurabh Kumar, Jacob Buckman, Ofir Nachum, and Marc~G Bellemare.
\newblock Deepmdp: Learning continuous latent space models for representation
  learning.
\newblock In \emph{International Conference on Machine Learning}, pages
  2170--2179, 2019.

\bibitem[Hafner et~al.(2019{\natexlab{b}})Hafner, Lillicrap, Fischer, Villegas,
  Ha, Lee, and Davidson]{hafner2019learning}
Danijar Hafner, Timothy Lillicrap, Ian Fischer, Ruben Villegas, David Ha,
  Honglak Lee, and James Davidson.
\newblock Learning latent dynamics for planning from pixels.
\newblock In \emph{International Conference on Machine Learning}, pages
  2555--2565, 2019{\natexlab{b}}.

\bibitem[Schrittwieser et~al.(2019)Schrittwieser, Antonoglou, Hubert, Simonyan,
  Sifre, Schmitt, Guez, Lockhart, Hassabis, Graepel, Lillicrap, and
  Silver]{schrittwieser2019mastering}
Julian Schrittwieser, Ioannis Antonoglou, Thomas Hubert, Karen Simonyan,
  Laurent Sifre, Simon Schmitt, Arthur Guez, Edward Lockhart, Demis Hassabis,
  Thore Graepel, Timothy Lillicrap, and David Silver.
\newblock Mastering atari, go, chess and shogi by planning with a learned
  model, 2019.

\bibitem[Oord et~al.(2018)Oord, Li, and Vinyals]{oord2018representation}
Aaron van~den Oord, Yazhe Li, and Oriol Vinyals.
\newblock Representation learning with contrastive predictive coding.
\newblock \emph{arXiv preprint arXiv:1807.03748}, 2018.

\bibitem[H{\'e}naff et~al.(2019)H{\'e}naff, Razavi, Doersch, Eslami, and
  Oord]{henaff2019data}
Olivier~J H{\'e}naff, Ali Razavi, Carl Doersch, SM~Eslami, and Aaron van~den
  Oord.
\newblock Data-efficient image recognition with contrastive predictive coding.
\newblock \emph{arXiv preprint arXiv:1905.09272}, 2019.

\bibitem[Wu et~al.(2018)Wu, Xiong, Stella, and Lin]{wu2018unsupervised}
Zhirong Wu, Yuanjun Xiong, X~Yu Stella, and Dahua Lin.
\newblock Unsupervised feature learning via non-parametric instance
  discrimination.
\newblock In \emph{2018 IEEE/CVF Conference on Computer Vision and Pattern
  Recognition (CVPR)}, pages 3733--3742. IEEE, 2018.

\bibitem[He et~al.(2019)He, Fan, Wu, Xie, and Girshick]{he2019momentum}
Kaiming He, Haoqi Fan, Yuxin Wu, Saining Xie, and Ross Girshick.
\newblock Momentum contrast for unsupervised visual representation learning.
\newblock \emph{arXiv preprint arXiv:1911.05722}, 2019.

\bibitem[Tian et~al.(2019)Tian, Krishnan, and Isola]{tian2019contrastive}
Yonglong Tian, Dilip Krishnan, and Phillip Isola.
\newblock Contrastive multiview coding.
\newblock \emph{arXiv preprint arXiv:1906.05849}, 2019.

\bibitem[Chen et~al.(2020)Chen, Kornblith, Norouzi, and Hinton]{chen2020simple}
Ting Chen, Simon Kornblith, Mohammad Norouzi, and Geoffrey Hinton.
\newblock A simple framework for contrastive learning of visual
  representations.
\newblock \emph{arXiv preprint arXiv:2002.05709}, 2020.

\bibitem[Belghazi et~al.(2018)Belghazi, Baratin, Rajeswar, Ozair, Bengio,
  Courville, and Hjelm]{belghazi2018mine}
Mohamed~Ishmael Belghazi, Aristide Baratin, Sai Rajeswar, Sherjil Ozair, Yoshua
  Bengio, Aaron Courville, and R~Devon Hjelm.
\newblock Mine: mutual information neural estimation.
\newblock \emph{arXiv preprint arXiv:1801.04062}, 2018.

\bibitem[Poole et~al.(2019)Poole, Ozair, Oord, Alemi, and
  Tucker]{poole2019variational}
Ben Poole, Sherjil Ozair, Aaron van~den Oord, Alexander~A Alemi, and George
  Tucker.
\newblock On variational bounds of mutual information.
\newblock \emph{arXiv preprint arXiv:1905.06922}, 2019.

\bibitem[Beattie et~al.(2016)Beattie, Leibo, Teplyashin, Ward, Wainwright,
  Küttler, Lefrancq, Green, Valdés, Sadik, Schrittwieser, Anderson, York,
  Cant, Cain, Bolton, Gaffney, King, Hassabis, Legg, and
  Petersen]{beattie2016deepmind}
Charles Beattie, Joel~Z. Leibo, Denis Teplyashin, Tom Ward, Marcus Wainwright,
  Heinrich Küttler, Andrew Lefrancq, Simon Green, Víctor Valdés, Amir Sadik,
  Julian Schrittwieser, Keith Anderson, Sarah York, Max Cant, Adam Cain, Adrian
  Bolton, Stephen Gaffney, Helen King, Demis Hassabis, Shane Legg, and Stig
  Petersen.
\newblock Deepmind lab, 2016.

\bibitem[Kim et~al.(2019)Kim, Kim, Jeong, Levine, and Song]{kim2019emi}
Hyoungseok Kim, Jaekyeom Kim, Yeonwoo Jeong, Sergey Levine, and Hyun~Oh Song.
\newblock Emi: Exploration with mutual information.
\newblock In \emph{International Conference on Machine Learning}, pages
  3360--3369, 2019.

\bibitem[Srinivas et~al.(2020)Srinivas, Laskin, and Abbeel]{srinivas2020curl}
Aravind Srinivas, Michael Laskin, and Pieter Abbeel.
\newblock Curl: Contrastive unsupervised representations for reinforcement
  learning.
\newblock \emph{arXiv preprint arXiv:2004.04136}, 2020.

\bibitem[Misra et~al.(2019)Misra, Henaff, Krishnamurthy, and
  Langford]{misra2019kinematic}
Dipendra Misra, Mikael Henaff, Akshay Krishnamurthy, and John Langford.
\newblock Kinematic state abstraction and provably efficient rich-observation
  reinforcement learning, 2019.

\bibitem[Bellemare et~al.(2017)Bellemare, Dabney, and
  Munos]{bellemare2017distributional}
Marc~G Bellemare, Will Dabney, and R{\'e}mi Munos.
\newblock A distributional perspective on reinforcement learning.
\newblock In \emph{Proceedings of the 34th International Conference on Machine
  Learning-Volume 70}, pages 449--458. JMLR. org, 2017.

\bibitem[Bellemare et~al.(2019)Bellemare, Roux, Castro, and
  Moitra]{bellemare2019distributional}
Marc~G Bellemare, Nicolas~Le Roux, Pablo~Samuel Castro, and Subhodeep Moitra.
\newblock Distributional reinforcement learning with linear function
  approximation.
\newblock \emph{arXiv preprint arXiv:1902.03149}, 2019.

\bibitem[Gutmann and Hyv{\"a}rinen(2010)]{gutmann2010noise}
Michael Gutmann and Aapo Hyv{\"a}rinen.
\newblock Noise-contrastive estimation: A new estimation principle for
  unnormalized statistical models.
\newblock In \emph{Proceedings of the Thirteenth International Conference on
  Artificial Intelligence and Statistics}, pages 297--304, 2010.

\bibitem[Mandelbaum et~al.(2007)Mandelbaum, Hlynka, and
  Brill]{mandelbaum2007nonhomogeneous}
Marvin Mandelbaum, Myron Hlynka, and Percy~H Brill.
\newblock Nonhomogeneous geometric distributions with relations to birth and
  death processes.
\newblock \emph{Top}, 15\penalty0 (2):\penalty0 281--296, 2007.

\bibitem[Bellman(1957)]{bellman1957markovian}
Richard Bellman.
\newblock A markovian decision process.
\newblock \emph{Journal of Mathematics and Mechanics}, pages 679--684, 1957.

\bibitem[Puterman(2014)]{puterman2014markov}
Martin~L Puterman.
\newblock \emph{Markov decision processes: discrete stochastic dynamic
  programming}.
\newblock John Wiley \& Sons, 2014.

\bibitem[Mazoure et~al.(2020)Mazoure, Doan, Li, Makarenkov, Pineau, Precup, and
  Rabusseau]{mazoure2020provably}
Bogdan Mazoure, Thang Doan, Tianyu Li, Vladimir Makarenkov, Joelle Pineau,
  Doina Precup, and Guillaume Rabusseau.
\newblock Provably efficient reconstruction of policy networks.
\newblock \emph{arXiv preprint arXiv:2002.02863}, 2020.

\bibitem[Levin and Peres(2017)]{levin2017markov}
David~A Levin and Yuval Peres.
\newblock \emph{Markov chains and mixing times}, volume 107.
\newblock American Mathematical Soc., 2017.

\bibitem[Song et~al.(2012)Song, Langfelder, and Horvath]{song2012comparison}
Lin Song, Peter Langfelder, and Steve Horvath.
\newblock Comparison of co-expression measures: mutual information,
  correlation, and model based indices.
\newblock \emph{BMC bioinformatics}, 13\penalty0 (1):\penalty0 328, 2012.

\bibitem[Von~Luxburg(2007)]{von2007tutorial}
Ulrike Von~Luxburg.
\newblock A tutorial on spectral clustering.
\newblock \emph{Statistics and computing}, 17\penalty0 (4):\penalty0 395--416,
  2007.

\end{thebibliography}
